\documentclass[11pt]{amsart}

\usepackage{microtype}
\usepackage{graphicx}
\usepackage{subfigure}
\usepackage{booktabs} 
\usepackage{multirow,multicol}
\usepackage{bm}
\usepackage{bcol}

\usepackage{hyperref}

\title{Safe Screening Rules for  $\ell_0$-Regression}
\author{Alper Atamt\"urk and Andr\'es G\'omez}
	
\thanks{ \noindent \hskip -5mm
	A. Atamt\"urk: Department of Industrial Engineering \& Operations Research, University of California, Berkeley, CA 94720.
	\texttt{atamturk@berkeley.edu}   \\
	A. G\'{o}mez: Daniel J. Epstein Department of Industrial \& Systems Engineering,  University of Southern California, CA 90089. \texttt{gomezand@usc.edu}
}

\usepackage{bbm}
\usepackage[mathlines,pagewise,left]{lineno}
\usepackage{amssymb}
\usepackage{color}
\usepackage{amsmath}
\usepackage{bcol}
\usepackage{multirow}
\usepackage{caption}
\usepackage{comment}
\usepackage[toc,page]{appendix}
\usepackage{graphicx} 
\usepackage{booktabs} 
\usepackage{pdflscape}
\usepackage[export]{adjustbox}

\usepackage{amsthm}
\newtheorem{proposition}{Proposition}

\theoremstyle{definition}

\theoremstyle{remark}
\newtheorem{remark}{Remark}

\newcommand{\R}{\ensuremath{\mathbb{R}}}
\newcommand{\REG}{\text{(REG) }}
\newcommand{\CARD}{\text{(CARD) }}
\newcommand{\MIP}{\text{(MIPR) }}
\newcommand{\MIPC}{\text{(MIPC) }}
\newcommand{\CQ}{\text{(CR) }}
\newcommand{\CQC}{\text{(CC) }}
\newcommand{\FD}{\text{(FDR) }}
\newcommand{\FDC}{\text{(FDC) }}
\newcommand{\zetaF}{\ensuremath{\zeta_{FR} }}
\newcommand{\zetaFC}{\ensuremath{\zeta_{FC} }}
\newcommand{\zetaC}{\ensuremath{\zeta_{CR}}}
\newcommand{\zetaCC}{\ensuremath{\zeta_{CC}}}

\usepackage[colorinlistoftodos,prependcaption,textsize=small]{todonotes}

\usepackage{natbib}
\bibpunct[, ]{(}{)}{,}{a}{}{,}

\begin{document}

\ignore{
\twocolumn[
\icmltitle{Safe Screening Rules for $\ell_0$-Regression}



\icmlsetsymbol{equal}{*}

\begin{icmlauthorlist}
\icmlauthor{Alper Atamt\"urk}{cal}
\icmlauthor{Andr\'es G\'omez}{usc}
\end{icmlauthorlist}

\icmlaffiliation{cal}{Department of Industrial Engineering and Operations Research, University of California, California, USA}
\icmlaffiliation{usc}{Department of Industrial and Systems Engineering, University of Southern California, Los Angeles, California, USA}

\icmlcorrespondingauthor{Alper Atamt\"urk}{atamturk@berkeley.edu}
\icmlcorrespondingauthor{Andr\'es G\'omez}{gomezand@usc.edu}

\icmlkeywords{Machine Learning, ICML}

\vskip 0.3in
]
}



\maketitle

\begin{abstract}
We give safe screening rules to eliminate variables from regression with $\ell_0$
regularization or cardinality constraint. These rules are based on guarantees that a feature may or 
may not be selected in an optimal solution. 
The screening rules can be computed from a convex relaxation solution in linear time, 
without solving the $\ell_0$ optimization problem. 
Thus, they can be used in a preprocessing step to safely remove variables from consideration apriori. 
Numerical experiments on real and synthetic data indicate that, 
on average, 76\% of the variables can be fixed to their optimal values, hence, reducing the computational burden 
for optimization substantially. Therefore, the proposed fast and effective screening rules extend the scope of
algorithms for $\ell_0$-regression to larger data sets.
\end{abstract}

\begin{center}
February 6, 2020
\end{center}

\BCOLReport{20.01}

\pagebreak

\section{Introduction}\label{sec:intro}
In machine learning and optimization communities, 
there is an increasing interest in regression models with $\ell_0$ and $\ell_2$ regularization:
\begin{align} 
&	\;\; \min_{x\in \R^n}
\|y-Ax\|_2^2+
\frac{1}{\gamma}\|x\|_2^2+\mu\|x\|_0, \text{ and}\label{eq:l0-l2}\tag{REG}\\
&	\;\; \min_{x\in \R^n}
\|y-Ax\|_2^2+
\frac{1}{\gamma}\|x\|_2^2\text{ s.t. }\|x\|_0\leq k,\label{eq:l0-l2_card}\tag{CARD}
\end{align}
where $A\in \R^{m\times n}$ is the model matrix, $y\in \R^m$ is the vector of response variables,
and $x\in \R^n$ is the vector of decision variables, i.e., regression coefficients to be estimated. 
Problem  \eqref{eq:l0-l2_card} has an explicit cardinality constraint on the number of non-zeros of $x$,
whereas \eqref{eq:l0-l2} is the regularized version of it. In these models,
the $\ell_0$ terms impose sparsity \citep{miller2002subset}, 
which is a necessity for large-dimensional model inference \citep{hastie2001elements,hastie2015statistical}, and
the $\ell_2$ (ridge) regularization \citep{hoerl1970ridge} imposes bias/shrinkage in the regression coefficients. The $\ell_2$ regularization can be interpreted, from the robust optimization perspective, as a correction term to account for uncertainty in the model matrix $A$ \citep{elghaoui1997robust,xu2009robustness}, and has been shown to improve the performance of sparse regression models in high-noise regimes \citep{mazumder2017subset}. 

The popular
$\ell_1$ \citep[lasso,][]{tibshirani1996regression} and $\ell_1$--$\ell_2$ \citep[elastic net,][]{zou2005regularization} regularizations  perform shrinkage and model selection simultaneously and, as convex proxies
for \eqref{eq:l0-l2}, they are very fast. However, thanks to substantial progress in the field of mixed-integer optimization (MIO), there is an increasing interest in solving the non-convex problems \eqref{eq:l0-l2}--\eqref{eq:l0-l2_card} directly. Indeed, several studies \citep{bertsimas2016best,cozad2014learning,gomez2018mixed,miyashiro2015subset,park2017subset} have shown that problems \eqref{eq:l0-l2}--\eqref{eq:l0-l2_card} with hundreds of variables can be solved to optimality simply by employing general purpose MIO solvers, and the resulting estimators outperform their $\ell_1$ counterparts. Nonetheless, solving the $\ell_0$ problems in this manner is orders-of-magnitude slower than solving the $\ell_1$ approximations and does not scale to problems with $n\geq 1$,$000$. Therefore, fast heuristics such as $\ell_1$ approximations, thresholding, 
local (but combinatorial) search algorithms or greedy methods 
\citep{hastie2017extended,hazimeh2018fast,xie2020scalable} may still be preferable in large-scale instances. 

The gap in the performance between exact methods for \eqref{eq:l0-l2}--\eqref{eq:l0-l2_card} and algorithms for a convex approximation is to be expected, as the $\ell_0$-regression is NP-hard. Moreover, there exist specialized software packages tailored to solving lasso and elastic net problems, such as \texttt{glmnet} \citep{friedman2010}, which include a variety of techniques specific to $\ell_1$ inference problems. 
In contrast, general purpose MIO solvers are not tailored to tackle \eqref{eq:l0-l2}--\eqref{eq:l0-l2_card}. 
Researchers have recently experimented with implementing branch-and-bound methods tailored for \eqref{eq:l0-l2}--\eqref{eq:l0-l2_card} \citep{bertsimas2017sparse,bertsimas2019sparse,dedieu2020learning,kimura2018minimization}, and the promising results indicate that there is substantial room for improvement for exact $\ell_0$-regression algorithms.

The purpose of this paper is to define \emph{screening rules} for nonconvex $\ell_0$-regression problems \eqref{eq:l0-l2}--\eqref{eq:l0-l2_card}. \citet{elghaoui2010safe} propose \emph{safe} rules for efficiently identifying regression variables that are guaranteed to be zero (null) in an optimal solution of the lasso problem, reducing the dimension of the problem to be solved a priori. \citet{tibshirani2012strong} subsequently propose \emph{strong} rules that may discard predictors that are part of an optimal lasso solution, but are quite effective in practice; these strong rules are incorporated into \texttt{glmnet}. Additional screening procedures have been proposed for other convex and lasso-type inference problems \citep{fercoq2015mind,ogawa2013safe,xiang2012fast,wang2013lasso,xiang2016screening}.  
To the best of our knowledge, no such screening rule is given to-date for the nonconvex $\ell_0$-regression problems \eqref{eq:l0-l2}--\eqref{eq:l0-l2_card}.

In MIO community, screening rules are used as part of preprocessing in branch-and-bound solvers \citep{atamturk2000conflict, martin-prep}. 
In contrast to convex optimization,  for MIO problems such as \eqref{eq:l0-l2}--\eqref{eq:l0-l2_card}, fixing a \emph{single} binary variable to zero reduces the number of feasible solutions by half; thus, the expected speedup of enumerative methods such as the branch-and-bound method is \emph{exponential in the number of variables} fixed. Therefore, effect of the screening rules on enumerative methods for non-convex optimization problems is significantly more than on polynomial-time algorithms for convex optimization problems.
Unfortunately, the existing screening rules in MIO solvers are tailored for linear mixed-integer problems and, as such, they are ineffective for \eqref{eq:l0-l2}--\eqref{eq:l0-l2_card}.

\subsection*{Contributions and outline} In this paper we propose \emph{safe screening rules} for nonconvex $\ell_0$-regression problems \eqref{eq:l0-l2}--\eqref{eq:l0-l2_card}. These rules can be applied to reduce the size of the problems, independent of the method used to
solve them.
Similar to the approach proposed by \citet{elghaoui2010safe} for lasso, the safe rules proposed are particularly effective in problems with large $\ell_0$--$\ell_2$ regularization terms, thus suitable for high noise regimes. The screening rules are obtained by exploiting convex perspective relaxations of the $\ell_0$ regression problems and using their Fenchel dual. The rules can be computed from a convex relaxation solution in \textit{linear time}, without having to solve the $\ell_0$ optimization problem. In our computational experiments with benchmark instances, the screening rules have been able to fix, on average, 76\% of the variables to their optimal values, and in some cases they have been sufficient to provably solve the problems outright. When used as preprocessing with a general purpose branch-and-bound solver, the screening procedure results in orders-of-magnitude speedups: \emph{instances previously requiring hours (or more) to prove optimality are solved in under 10 seconds with screening}. Consequently, 
the speed and effectiveness of the safe screening rules extend the scope of algorithms for $\ell_0$-regression problems to larger data sets.

The rest of the paper is organized as follows. In Section~\ref{sec:mip} we describe mixed-integer formulations and convex perspective relaxations of problems \eqref{eq:l0-l2}-\eqref{eq:l0-l2_card}. In Section~\ref{sec:screen}, we derive the safe screening rules for \REG and \CARD based on Fenchel duality of the perspective relaxations. In Section~\ref{sec:comp}, we present our computational experiments with synthetic and real benchmark instances from the literature. We conclude in Section~\ref{sec:conc} with a few final remarks.

\section{Mixed-integer \& perspective formulations}
\label{sec:mip}

Introducing indicator variables $z \in \{0,1\}^n$ such that $z_i= 0\implies x_i=0$, problem 
\eqref{eq:l0-l2} can be naturally formulated as the quadratic mixed-integer optimization problem 
\begin{subequations}\label{eq:mip0}
	\begin{align}
	\min_{x,z}\;&  \|y-Ax\|_2^2 +\frac{1}{\gamma}\sum_{i=1}^n {x_i^2} +\mu\sum_{i=1}^nz_i\\
	\quad			\text{s.t.}\;&x_i(1-z_i)=0, \quad \quad i=1,\ldots,n\label{eq:mip0_compl}\\
	&x \in \R^n, \ z\in \{0,1\}^n.\label{eq:mip0_int}
	\end{align}
\end{subequations}
For each $i$, the complementarity constraint $x_i(1-~z_i)=0$, ensures that $x_i=0$ whenever $z_i=0$. Such complementary constraints can be linearized via ``big-$M$" constraints $|x_i|\leq Mz_i$ \citep{bertsimas2016best} for a suitably large value of $M$. However, such formulations with large values of $M$ are weak and may lead to poor performance as a consequence. A stronger formulation
can be given by utilizing the perspective of the univariate quadratic function $x_i^2$: 
\begin{subequations}\label{eq:mip}
	\begin{align}
	\zeta_R	= \min_{x,z}\;&  \|y-Ax\|_2^2 +\frac{1}{\gamma}\sum_{i=1}^n\frac{x_i^2}{z_i} +\mu\sum_{i=1}^nz_i\\
	\MIP \quad			\text{s.t.}\;&x_i(1-z_i)=0, \quad \quad i=1,\ldots,n\label{eq:mip_compl}\\
	&x \in \R^n, \ z\in \{0,1\}^n,\label{eq:mip_int}
	\end{align}
\end{subequations}
where we adopt the convention that $x_i^2/z_i=0$ if $z_i=x_i=0$, and $x_i^2/z_i=+\infty$ if $z_i=0$ and $x_i\neq 0$.
The perspective function $x_i^2/z_i$ significantly strengthens the convex relaxation and can be formulated with 
conic quadratic constraints \cite{akturk2009strong,dong2015regularization,frangioni2006perspective,gunluk2010perspective,xie2020scalable}. The perspective formulation is also at the core of recent specialized branch-and-bound methods for sparse regression \cite{bertsimas2017sparse,bertsimas2019sparse}. 
A similar strong mixed-integer formulation of \eqref{eq:l0-l2_card} is 
\begin{subequations}\label{eq:mip-card}
	\begin{align}
	\zeta_C	= \min_{x,z}\;&  \|y-Ax\|_2^2 +\frac{1}{\gamma}\sum_{i=1}^n\frac{x_i^2}{z_i} \\
	\MIPC \quad			\text{s.t.}\; &
	\sum_{i=1}^n z_i\leq k\\
	&x_i(1-z_i)=0, \quad \quad i=1,\ldots,n\label{eq:mip-card_compl}\\
	&x \in \R^n, \ z\in \{0,1\}^n.\label{eq:mip-card_int}
	\end{align}
\end{subequations}

Convex relaxation of the mixed-integer programs are obtained by dropping complementary constraints \eqref{eq:mip_compl} and \eqref{eq:mip-card_compl}, and relaxing the integrality constraints in \eqref{eq:mip_int} and \eqref{eq:mip-card_int} to $z\in [0,1]^n$. Thus, we obtain the convex relaxation
\begin{subequations}\label{eq:cq}
	\begin{align}
	\ \zetaC =  \min_{x,z}\;&  \|y-Ax\|_2^2 +\frac{1}{\gamma}\sum_{i=1}^n\frac{x_i^2}{z_i} + \mu\sum_{i=1}^nz_i\\
	\CQ\quad\quad\quad&x \in \R^n, \ z\in [0,1]^n,\label{eq:cq_bdns}
	\end{align}
\end{subequations}
of \MIP, and the convex relaxation 
\begin{subequations}\label{eq:cq-card}
	\begin{align}
	\ \zetaCC =  \min_{x,z}\;&  \|y-Ax\|_2^2 +\frac{1}{\gamma}\sum_{i=1}^n\frac{x_i^2}{z_i} \\
	\CQC \quad \quad \quad	& 	\sum_{i=1}^n z_i\leq k,\ x \in \R^n, \ z\in [0,1]^n\label{eq:cq-card_bdns}
	\end{align}
\end{subequations}
of \MIPC.

The optimal solutions of \eqref{eq:cq} and \eqref{eq:cq-card} are good statistical estimators on their own right. Indeed, \citet{pilanci2015sparse} propose convex relaxations of \eqref{eq:l0-l2}--\eqref{eq:l0-l2_card}, which are later shown to be equivalent to perspective relaxations \citep{xie2020scalable}, and study their strength and conditions for delivering optimal solutions.

\section{Safe screening rules for \REG \& \CARD}
\label{sec:screen}

In this section, we give safe screening rules for  problems \MIP and \MIPC,
to fix the binary indicator variables at their optimal values before solving them. The screening rules require an upper bound on the optimal objective value of the mixed-integer optimization problems \MIP or \MIPC and an optimal solution of the perspective relaxation \CQ or \CQC, respectively. 

\begin{proposition}[Safe screening rules for \ref{eq:l0-l2}] \label{prop:screen-reg}
	Let $x^*$ be an optimal solution to \CQ with objective value $\zetaC$, $\varepsilon^*=y-Ax^*$,  $\delta_i=\left(A_i'\varepsilon^*\right)^2$,  $i=1, \ldots,n$, and let 
	$\bar \zeta$ be an upper bound on $\zeta_R$. Then any optimal solution to \MIP satisfies, 	
	\[
	z_i = 
	\begin{cases}0, & \text{ if }  \zetaC+\mu-\gamma\delta_i> \bar\zeta \\
	1, & \text{ if } \zetaC- \mu+\gamma\delta_i> \bar\zeta.
	\end{cases}
	\]

\end{proposition}

\begin{proposition}[Safe screening rules for \ref{eq:l0-l2_card}] \label{prop:screen-card}
	Let $x^*$ be an optimal solution to \CQC with objective value $\zetaCC$, $\varepsilon^*=y-Ax^*$,  $\delta_i=\left(A_i'\varepsilon^*\right)^2$,  $i=1, \ldots,n$, $\delta_{[k]}$ be the $k$-th largest value of vector $\delta$, and let 
	$\bar \zeta$ be an upper bound on $\zeta_C$. Then any optimal solution to \MIPC satisfies, 
	\[
	z_i = 
	\begin{cases}0, &  \!\!\!\! \text{if }  \delta_i\leq \delta_{[k+1]}\text{ and } \ \zetaCC-\gamma(\delta_i-\delta_{[k]}) > \bar\zeta \\
	1, & \!\!\!\! \text{if } \delta_i\geq \delta_{[k]}\text{ and } \ \zetaCC+\gamma(\delta_i-\delta_{[k+1]})> \bar\zeta.
	\end{cases}
	\]
\end{proposition}

We prove Propositions~\ref{prop:screen-reg} and \ref{prop:screen-card} using Fenchel duality in \S\ref{sec:derivation}. Before doing so, in \S\ref{sec:cost}, we discuss the computational cost of implementing the screening rules.

\subsection{Computational cost}\label{sec:cost}

Computing optimal solutions to the convex perspective relaxations can be done in polynomial time, while finding upper bounds for the non-convex mixed-integer optimization can be accomplished via fast heuristics, thus the screening rules require substantially less time than solving \eqref{eq:l0-l2}-\eqref{eq:l0-l2_card}  to optimality. In this section we give pointers on how to do so effectively, and argue that in the context of branch-and-bound methods the overhead of the screening rules is linear in $n$.

\subsubsection*{Solving perspective relaxations} Formulations \CQ and \CQC can be conveniently solved using off-the-shelf conic quadratic solvers \citep{akturk2009strong,gunluk2010perspective} --- this is the approach we use here. \citet{pilanci2015sparse} use a projected quasi-Newton method to solve \CQC which, they argue, is comparable in complexity to the lasso for low values of $k$. \citet{bertsimas2017sparse,bertsimas2019sparse} use a linear outer approximation method which they report performs faster than the lasso.

In fact, mixed-integer optimization methods based on formulations \MIP \ or \MIPC will solve problems \CQ or \CQC at the root node of the branch-and-bound tree anyway. Thus, in this context, an optimal solution of the perspective relaxation can be obtained without an additional cost. 

\subsubsection*{Obtaining upper bounds} There exist extensive work on heuristics for sparse regression, including stepwise selection methods \citep{efroymson1966stepwise} and other methods mentioned in \S\ref{sec:intro}. Branch-and-bound methods, both based on off-the-shelf solvers or recent specialized implementations, use heuristics to warm-start the solvers and may even require them to initialize big-$M$ values \citep{bertsimas2016best,dedieu2020learning}. Thus, upper bounds in this context are available without incurring in additional costs.

In addition, feasible solutions for sparse regression problems can be obtained directly from convex relaxations. For example, \citet{pilanci2015sparse} use randomized rounding to obtain high quality feasible solutions of perspective relaxations. In our computations with cardinality constrained problems, we use a simpler rounding mechanism informed by Proposition~\ref{prop:screen-card}: given an optimal solution for \CQC, we set $z_i=1$ for the $k$ largest values of $\delta$ (breaking ties arbitrarily), and set $x$ equal to the least squares estimator corresponding to the chosen variables.

\subsubsection*{Additional operations} It is easy to see that for problem \eqref{eq:l0-l2},
given a convex relaxation solution and upper bound, the screening rule of Proposition~\ref{prop:screen-reg} can be computed in $O(n)$ time
with a single pass along the variables. 
For \eqref{eq:l0-l2_card}, given a convex relaxation solution and upper bound, $\delta_{[k]}$ and $\delta_{[k+1]}$ can be selected in $O(n)$ (without the need for sorting) and then the screening rule of
Proposition~\ref{prop:screen-card} can be computed in $O(n)$ time as well.

\subsection{Derivation of the screening rules}\label{sec:derivation}

We now derive the screening rules using Fenchel duality. Note that, whereas \citet{pilanci2015sparse} and \citet{bertsimas2017sparse} derive their methods based on the Fenchel dual of the error term $\|y-Ax\|_2^2$, we instead use the dual of the perspective terms.

\subsubsection{Derivation of Proposition~\ref{prop:screen-reg}}\label{sec:derivation_reg}

Let $h^*(p,q)$ be the bivariate convex conjugate of the perspective function $x^2/z$, i.e.,
\begin{align} \label{eq:fenchel} 
h^*(p,q)=\max_{x,z}px+qz-\frac{x^2}{z} \cdot
\end{align}
From Fenchel's inequality, we have
\begin{align} \label{eq:fenchelineq} 
px+qz-h^*(p,q)\leq h(x,z)
\end{align}
for any $p,q, x,z \in \R$. Employing \eqref{eq:fenchelineq} for each term to get a lower bound on \CQ and maximizing the lower bound, we obtain the Fenchel dual for \eqref{eq:mip}:
\begin{subequations}\label{eq:dual}
	\begin{align}
	\max_{p,q\in \R^n}\min_{x,z}\;&  \|y-Ax\|_2^2 +\mu\sum_{i=1}^nz_i \\
	&+ \frac{1}{\gamma}\sum_{i=1}^n\Big(p_ix_i+q_iz_i-h^*(p_i,q_i)\Big)\\
	\text{s.t.}\;& x \in \R^n, \ z\in [0,1]^n.
	\end{align}
\end{subequations}

Indeed, the conjugate
function $h^*$ can be computed in closed form. Since \eqref{eq:fenchel} is concave in both $x$ and $z$,
by taking derivatives with respect to $x$ and $z$ and setting to zero, we find the optimality conditions:
\begin{align}
p-\frac{2x}{z} = 0 \label{eq:derivativeX}\\
q+\left(\frac{x}{z}\right)^2 = 0,\label{eq:derivativeY}
\end{align}
since, otherwise, \eqref{eq:fenchel} is unbounded. The optimality conditions imply that
\begin{align*}
\frac{p^2}{4}=-q
\text{ and }  px+qz-\frac{x^2}{z}=0,
\end{align*}
where the second inequality is obtained by multiplying \eqref{eq:derivativeX} by $x$ and \eqref{eq:derivativeY} by $y$, and summing them up. Thus, 
$$h^*(p,q)=\begin{cases}0, & \text{if }q=-p^2/4\\
+\infty, & \text{otherwise.}\end{cases}
$$

Therefore, we find that \eqref{eq:dual} reduces to
\begin{subequations}\label{eq:dual_explicit}
	\begin{align}
	\zetaF		=	\max_{p\in \R^n}\min_{x,z}\;&  \|y-Ax\|_2^2  +\mu\sum_{i=1}^nz_i \\
	\FD \quad \quad \quad \quad	\ \ &	+\frac{1}{\gamma}\sum_{i=1}^n\Big(p_ix_i-\frac{p_i^2}{4}z_i\Big) \\
	\text{s.t.}\;& x \in \R^n, \ z\in [0,1]^n.
	\end{align}
\end{subequations}
In fact, if $\max$ and $\min$ are interchanged in \eqref{eq:dual_explicit}, then $p_i^*=2\frac{x_i}{z_i}$ (if $x_i$ and $z_i$ are both non-zero) and we recover precisely \CQ; thus, there is no duality gap between \CQ and \FD and we have 
$\zetaC  = \zetaF$.

In optimal solutions of the inner minimization problem we have 
$$
z_i =\begin{cases}0, & \text{if }  \mu-\frac{p_i^2}{4\gamma}> 0 \\
1, & \text{if } \mu-\frac{p_i^2}{4\gamma}< 0\\
\in [0,1] & \text{otherwise}.\end{cases}
$$
and $A'Ax=A'y-\frac{1}{2\gamma} p$. Note that if $\mu-\frac{p_i^2}{4\gamma}\neq 0$ for all $i=1,\ldots,n$, then the optimal solution of the inner minimization problem in \FD is unique; in this case, by strong duality, that solution is also optimal for \CQ and, since it is integral, it is in fact optimal for \MIP as well. However, if $\mu-\frac{p_i^2}{4\gamma}= 0$ for some $i$, then the inner minimization problem in \FD has an infinite number of optimal solutions and the solution of \CQ may not be integral.

Now, let $x^*$ be an optimal solution of \CQ and $\varepsilon^*=y-Ax^*$ be the vector of residuals. 
Given $x^*$, a corresponding optimal dual solution $p^*$ can be recovered as $A'Ax^*=A'y-\frac{1}{2\gamma} p^*$, or $p^*={2\gamma}A'\varepsilon^*$. Moreover, we find that 
$$\mu-\frac{(p_i^*)^2}{4\gamma}=\mu-{\gamma(A_i'\varepsilon^*)^2}=\mu-\gamma\delta_i,$$ 
where $A_i$ is the $i$-th column of $A$. Consequently, optimal $(p^*, z^*)$  for \FD can be recovered from $\varepsilon^*$. 
We can now give the proof of Proposition~\ref{prop:screen-reg}.

\begin{proof}[Proof of Proposition~\ref{prop:screen-reg}]
	Suppose $\mu-\gamma \delta_i>0$ and thus $z_i = 0$ in an optimal solution to \FD. Note that in this case the inequality  $\zetaC- \mu+\gamma\delta_i> \bar\zeta$ is never satisfied. Let $\zetaF(z_i=1)$ be the optimal objective value of the Fenchel dual with the additional constraint $z_i=1$. Note that
	$$ \zetaF+\mu-\gamma\delta_i=\zetaF+\mu-{(p_i^*)^2}/4\gamma \le \zetaF(z_i=1),$$
	and the inequality is tight if the dual variables $p^*$ are still optimal optimal after introducing the constraint $z_i=1$. 
	Thus, if
	$ \zetaF+\mu-\gamma\delta_i > \bar \zeta$, we conclude that any feasible solution for \CQ with $z_i=1$ has an objective worse than the upper bound and, in particular, there exists no optimal solution of \MIP with $z_i=1$. 
	
	Similarly, suppose $\mu-\gamma \delta_i<0$ and $z_i = 1$ in an optimal solution to \FD. Since
	$ \zetaF-\mu+\gamma\delta_i \le \zetaF(z_i=0)$, if the lower bound
	$ \zetaF+\mu-{(p_i^*)^2}/4\gamma > \bar \zeta$, we conclude that there exists no optimal MIP solution with $z_i=0$.
\end{proof}

\begin{remark}
	If $A'A$ is invertible, then an explicit formulation of the dual problem \eqref{eq:dual_explicit} can be obtained as
	\begin{align*}
	\max_{p\in \R^n}\;&\|y\|_2^2-\left(A'y-\frac{1}{2\gamma} p\right)'(A'A)^{-1} \bigg (A'y-\frac{1}{2\gamma} p \bigg) \\
	&		+\sum_{i=1}^n\min\left\{0,\mu-\frac{p_i^2}{4\gamma}\right\}.
	\end{align*}
\end{remark}

\subsubsection{Derivation of Proposition~\ref{prop:screen-card}}\label{sec:derivation_card}

Using identical arguments as in \S\ref{sec:derivation_reg}, we find the Fenchel dual of \CQC as
\begin{subequations}\label{eq:dualCard_explicit}
	\begin{align}
	\zetaFC \!  = \!	\max_{p\in \R^n}\min_{x,z}\;&  \|y \!-\!Ax\|_2^2 +\! \frac{1}{\gamma}\sum_{i=1}^n\! \Big(p_ix_i\!-\!\frac{p_i^2}{4}z_i\Big)\\
	\FDC \quad \quad 	\text{s.t.}  &\sum_{i=1}^n z_i\leq k,\; x \in \R^n,  z\in [0,1]^n.
	\end{align}
\end{subequations}

As for \FD if $\max$ and $\min$ are interchanged, then $p_i^*=2\frac{x_i}{z_i}$ (if $x_i$ and $z_i$ are both non-zero) and we recover precisely \CQC; thus, there is no duality gap between \CQC and \FDC and we have 
$\zetaCC  = \zetaFC$.

Observe that for the inner minimization problem, an optimal solution satisfies
$z_i=1$ for indices with the largest $k$ values of 
$\frac{p_i^2}{4\gamma}$ and $z_i=0$ otherwise. Moreover, if there is no tie between the $k$-th and $(k+1)$-st largest value in an optimal solution of \FDC, then this solution is unique and is also optimal\footnote{A similar result is given in \citep[][Prop. 1]{pilanci2015sparse}.} for \CQC and \MIPC. Otherwise, if there is a tie, then \CQC may not have optimal solutions integral in $z$. 

Now, let $x^*$ be an optimal solution of the convex relaxation of \MIPC, and let $\varepsilon^*=y-Ax^*$ be the vector of residuals. Then, the corresponding optimal dual solution $p^*$ can be recovered as $A'Ax^*=A'y-\frac{1}{2\gamma} p^*$, or $p^*={2\gamma}A'\varepsilon^*$. Moreover, we find that $$-\frac{(p_i^*)^2}{4\gamma}=-\gamma{(A_i'\varepsilon^*)^2}=-\gamma \delta_i.$$ 

\begin{proof}[Proof of Proposition~\ref{prop:screen-card}]
	Suppose $\delta_i\leq \delta_{[k+1]}$. Then,
	$z_i = 0$ in an optimal solution of the inner minimization in \FDC; let $z_{[k]}$ be the indicator variables corresponding to the term $\delta_{[k]}$. Let $\zetaFC(z_i=1)$ be the optimal objective value of Fenchel dual with the additional constraint $z_i=1$. The cardinality constraint implies that $z_{[k]}=0$ for an optimal solution of this problem. 
	Since
	$\zetaCC-\gamma\delta_i  +\delta_{[k]} \le \zetaFC(z_i=1)$, if the lower bound
	$ \zetaCC-\gamma\delta_i  +\delta_{[k]} > \bar \zeta$, we conclude that there exists no optimal solution to \MIPC with $z_i=1$.
	
	Similarly, suppose
	$\delta_i\geq \delta_{[k]}$; then, we have
	$z_i = 1$ in an optimal solution of the inner minimization of \FD. 
	Let $\zetaFC(z_i=0)$ be the objective value of the Fenchel dual with the additional constraint $z_i=0$. 
	Since
	$ \zetaCC+\gamma\delta_i - \delta_{[k+1]} \le \zetaFC(z_i=0)$, if the lower bound
	$ \zetaCC+\gamma\delta_i - \delta_{[k+1]}  > \bar \zeta$, 
	we conclude that there exists no optimal solution to \MIPC with $z_i=0$.
\end{proof}

\section{Computational experiments}
\label{sec:comp}

In this section we report on our computational experiments to test the effectiveness of 
the screening rules for the cardinality constrained sparse regression problem \CARD.
As the statistical merits of solving \CARD 
are, by now, extensively documented in the literature \citep{atamturk2019rank-one,bertsimas2016best,bertsimas2017sparse,bertsimas2019sparse,hastie2017extended,hazimeh2018fast,mazumder2017subset}, we focus on the impact of the safe screening rules on solving \MIPC efficiently. In our computations we use CPLEX 12.8 mixed-integer optimizer. All experiments are performed on a laptop with eight Intel(R) Core(TM) i7-8550 CPUs and 16GB RAM. In \S\ref{sec:synt} we test the screening rules on ``standard" synthetic data sets \citep{atamturk2019rank-one,bertsimas2016best,bertsimas2019sparse,hastie2017extended,xie2020scalable}, and in \S\ref{sec:real} we use the real data sets reported in Table~\ref{tab:datasets}. The ``Diabetes" data set is first used by \citet{efron2004least}, whereas the other data sets are obtained from the UCI Machine Learning Repository \citep{dua2019}.

\begin{table}[!h]
	\begin{center}
		\small
		\caption{Real data sets used.}
		\label{tab:datasets}
		\begin{tabular}{ l |c c}
			\hline
			\textbf{Name} & $\bm{n}$ & $\bm{m}$\\
			\hline
			Diabetes &64 &442\\
			Crime & 100 & 1993\\
			Parkinsons & 753 & 756\\
			CNAE & 856 & 1,081\\
			Micromass & 1,300 & 360\\
			\hline
		\end{tabular}
		
	\end{center}
\end{table}

\subsection{Synthetic data}\label{sec:synt}

We follow the data generation methodology of \citet{bertsimas2019sparse}, where instances are generated according to a number of features of $n$, number of rows $m$, true sparsity $k$, regularization parameter $\gamma$, autocorrelation parameter $\rho$, and signal noise ratio (SNR).
In our experiments, we let $n=1$,$000$, $m=500$, $k\in \{10,30,50\}$, $\gamma=2^i\gamma_0$ with $i\in \{-1,0,2,4\}$ and $\gamma_0=\frac{n}{mk\max_{i}\|a_i\|_2^2}$ (where $a_i$ denotes the $i$-th row of $A$), $\rho\in \{0.2,0.5,0.7\}$, and $\text{SNR}\in \{0.05,1.00,6.00\}$. The parameters $m$, $\gamma$, $\rho$ and SNR coincide with the values used in \citet{bertsimas2019sparse}. 
Our instances are smaller with $n=1,000$ and $k\in \{10,30,50\}$ as we use a general purpose mixed-integer solver rather than a
tailored solution method for \MIPC as in \citet{bertsimas2019sparse}. 
Several other papers in the literature generate data similarly.
Finally, we set the time limit to ten minutes. 

\begin{figure}[!h]
	\begin{center}
		\includegraphics[trim={10.8cm 6cm 10.8cm 6cm},clip,width=0.8\columnwidth]{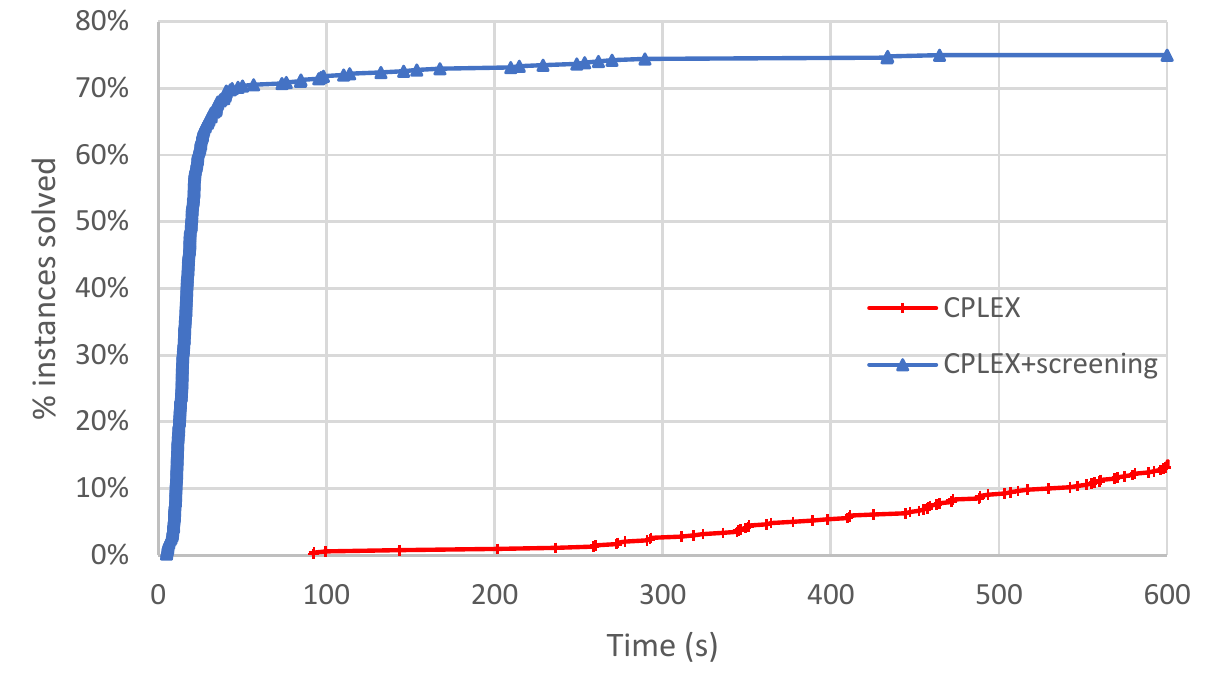}
	\end{center}
	\caption{Number of instances solved as a function of the time (sec.)} 
	\label{fig:performance}
\end{figure}

\begin{figure}[!h]
	\begin{center}
		\includegraphics[trim={10.8cm 6cm 10.8cm 6cm},clip,width=0.8\columnwidth]{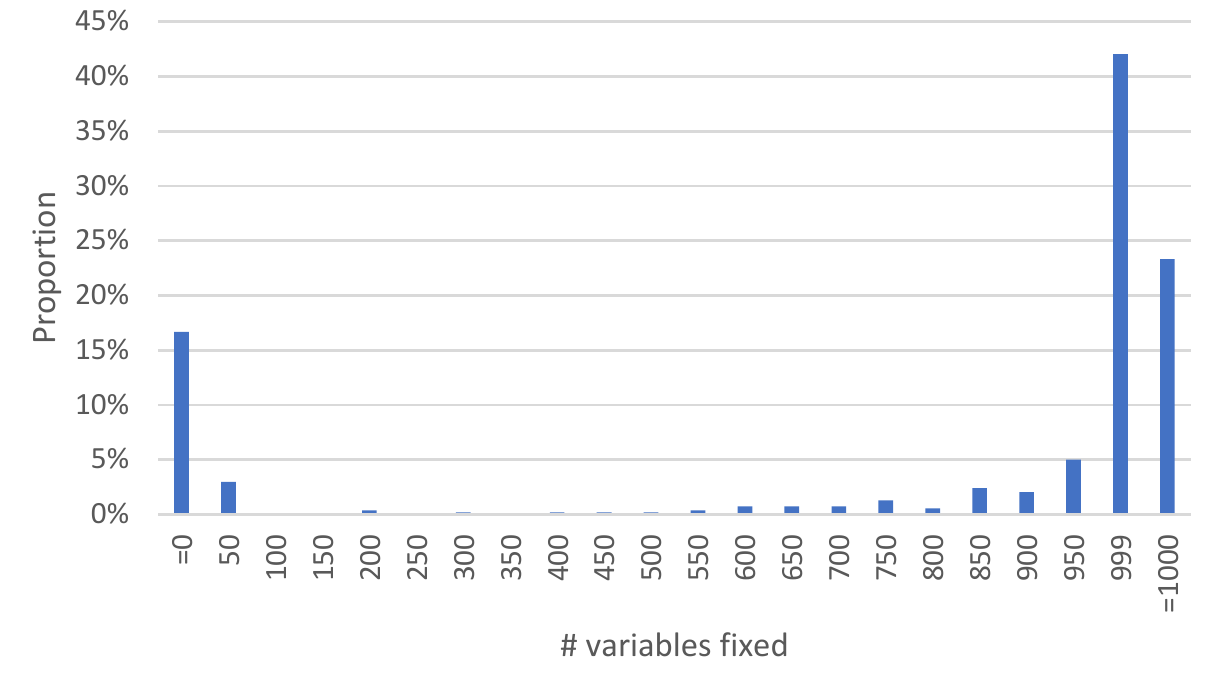}
	\end{center}
	\caption{Distribution of the number of variables fixed on synthetic instances.}
	\label{fig:distribution}
\end{figure}

Figures \ref{fig:performance} and~\ref{fig:distribution}  show aggregated results over all 540 synthetic instances tested.
Figure~\ref{fig:performance} depicts the performance profiles of CPLEX with and without the safe screening rules proposed in the paper. We see that default CPLEX struggles with instances of this size, and is able to solve only 14\% of the instances within the time limit; similar 
performance for general purpose MIP solvers has been observed in the literature for instances with $n=1,000$  \citep{hastie2017extended,xie2020scalable}. In contrast, when the screening rules are incorporated, the performance improves substantially: it only takes 11 seconds to solve the same 14\% of the instances, and 75\% of the instances are provably solved to optimality within the ten-minuted time limit. Thus, for the synthetic instances that are solved to optimality by both methods, \emph{the screening procedure results in a $60\times$ speedup}. In fact, as Figure~\ref{fig:distribution} shows, the screening procedures alone are sufficient to prove optimality for 23\% of the instances, and are able to fix 75\% or more of the variables in an additional 52\% of the instances. There is, however, a small portion of the instances where few or no variables were fixed by the screening procedure. 

Table~\ref{tab:fixed} presents detailed information on the number of variables fixed as a function of the parameters $k$, $\gamma$, $\rho$, and SNR. Each entry in the table corresponds to an average over five identically generated instances.
As the parameter $k$ decreases (imposing higher $\ell_0$ regularization) and the parameter $\gamma$ increases (imposing higher $\ell_2$ regularization), the screening procedures become more effective at fixing variables. We also observe that the screening rules are more effective when the signal-noise ratio is large, while the parameter $\rho$ plays a relatively minor role.

\begin{table}[!tb]
	\begin{center}
		\small
		\setlength{\tabcolsep}{2pt}
		\caption{Number of variables fixed in synthetic instances with $n=1,000$.}
		\label{tab:fixed}
		\begin{tabular}{ c c c|c c c | c c c | c c c | c}
			\hline
			&& $\bm{k}$& \multicolumn{3}{c|}{$\bm{10}$} & \multicolumn{3}{c|}{$\bm{30}$}& \multicolumn{3}{c|}{$\bm{50}$}&\multirow{2}{*}{\textbf{Average}}\\
			$\bm{\gamma}$&\textbf{SNR}&$\bm{\rho}$&$\bm{.2}$&$\bm{.5}$&$\bm{.7}$&$\bm{.2}$&$\bm{.5}$&$\bm{.7}$&$\bm{.2}$&$\bm{.5}$&$\bm{.7}$&\\
			\hline
			\multirow{3}{*}{$\bm{2^{-1}\gamma_0}$}&$\bm{0.05}$&&995&996&993&997&791&983&996&792&934&\multirow{3}{*}{$\bm{930\pm 219}$}\\
			&$\bm{1.00}$&&1,000&998&999&997&912&473&906&982&950&\\
			&$\bm{6.00}$&&1,000&1,000&997&993&988&987&1000&752&692&\\
			\hline
			
			\multirow{3}{*}{$\bm{2^{0}\gamma_0}$}&$\bm{0.05}$&&983&986&991&980&972&988&988&989&800&\multirow{3}{*}{$\bm{967\pm 147}$}\\
			&$\bm{1.00}$&&998&997&996&958&977&988&973&789&994\\
			&$\bm{6.00}$&&1,000&999&995&997&993&997&785&997&992\\
			\hline
			
			\multirow{3}{*}{$\bm{2^{2}\gamma_0}$}&$\bm{0.05}$&&553&245&621&893&640&751&804&952&902&\multirow{3}{*}{$\bm{886\pm 232}$}\\
			&$\bm{1.00}$&&991&988&977&971&969&940&968&976&962\\
			&$\bm{6.00}$&&1,000&1,000&983&978&980&974&962&975&968\\
			\hline
			
			\multirow{3}{*}{$\bm{2^{4}\gamma_0}$}&$\bm{0.05}$&&0&0&0&0&0&0&1&111&0&\multirow{3}{*}{$\bm{276\pm 410}$}\\
			&$\bm{1.00}$&&577&194&174&302&455&457&40&166&144\\
			&$\bm{6.00}$&&1,000&999&597&939&379&109&40&297&471&\\
			\hline
			\multicolumn{3}{c|}{\textbf{Average}}&\multicolumn{3}{c|}{$\bm{801\pm 385}$}&\multicolumn{3}{c|}{$\bm{770\pm 377}$}&\multicolumn{3}{c|}{$\bm{723\pm 409}$}&\textbf{$\bm{765\pm 391}$}\\
			\hline
		\end{tabular}
	\end{center}
\end{table}

\subsection{Real data} \label{sec:real}
We test the safe screening procedure in the data sets given in Table~\ref{tab:datasets}.  For each data set, we solve problem \MIPC with 
$k\in \{10,20,30\}$.  \citet{bertsimas2019sparse} indicate in the documentation of their code\footnote{\url{https://github.com/jeanpauphilet/SubsetSelectionCIO.jl}.} that setting $\gamma=1/\sqrt{m}$ is an appropriate scaling for regression problems. For this value of $\gamma$, on average, 98.2\% of the variables are fixed by the screening procedure, and all instances are solved in four seconds. To better understand the effectiveness of the screening procedures for a broader set of parameters, we let $\gamma=2^i\gamma_0$ with $i\in\{-1,0,1,2,3,4,5,6,7,8\}$ and $\gamma_0$ as described in \S\ref{sec:synt}.

\begin{figure}[!h]
	\begin{center}
		\includegraphics[trim={10.8cm 6cm 10.8cm 6cm},clip,width=0.8\columnwidth]{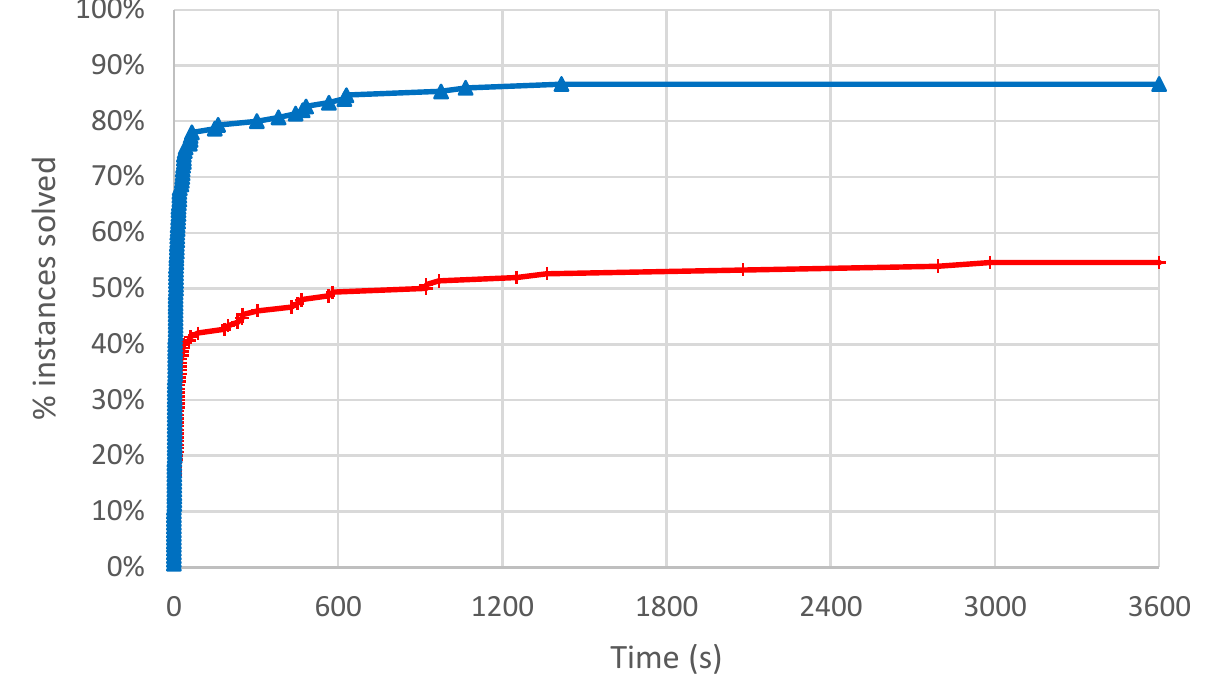}
	\end{center}
	\caption{Number of instances solved as a function of the time (sec.)} 
	\label{fig:performanceReal}
\end{figure}

\begin{figure}[!h]
	\begin{center}
		\includegraphics[trim={10.8cm 6cm 10.8cm 6cm},clip,width=0.8\columnwidth]{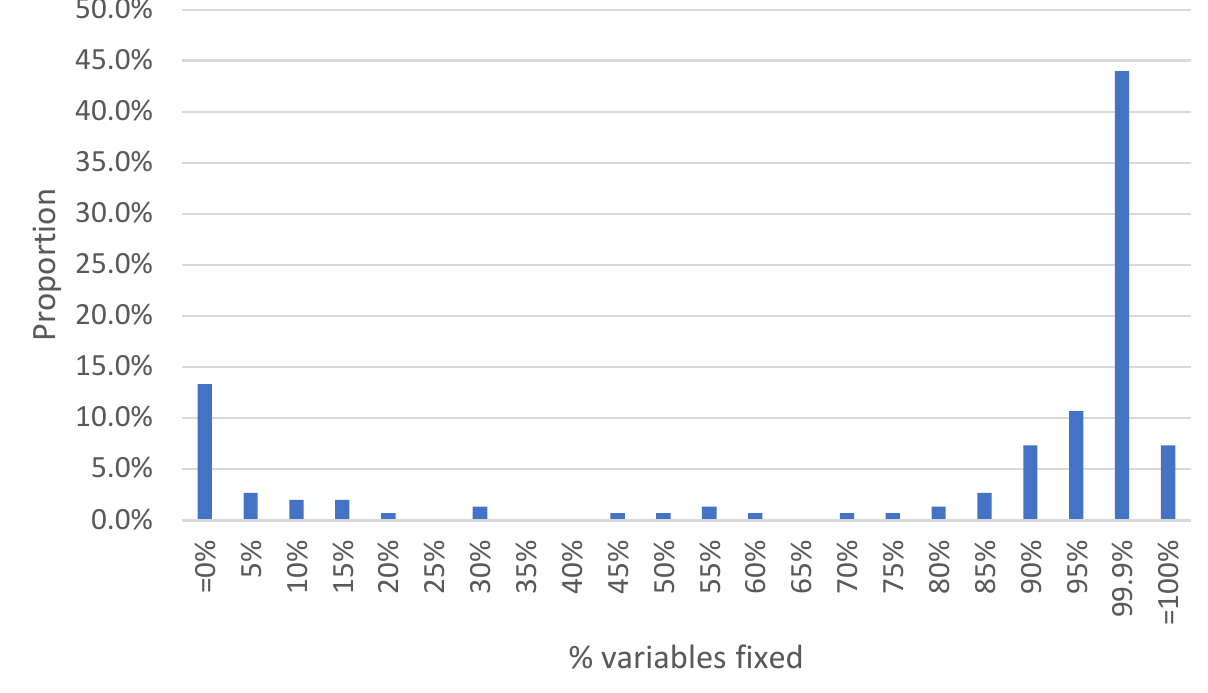}
	\end{center}
	\caption{Distribution of the number of variables fixed on instances with real data.}
	\label{fig:distributionReal}
\end{figure}

Figures \ref{fig:performanceReal} and \ref{fig:distributionReal}  
display the aggregated results over 150 instances tested with a time limit of one hour.
The performance profile in Figure~\ref{fig:performanceReal} shows that default CPLEX is able to solve 55\% of the instances in one hour. When the screening rules are incorporated, the same 55\% of the instances are solved in under 10 seconds, and 87\% of the instances are solved within the time limit of one hour. Therefore, for the instances that are solved to optimality by both methods, \emph{the screening procedure results in a $360 \times$ speedup}. The distribution of the percentage of variables fixed (Figure~\ref{fig:distributionReal}) is similar to the one reported in \S\ref{sec:synt}, and 75\% or more of the variables are fixed in 73\% of the instances.

Figure~\ref{fig:fixed_real} depicts the number of variables fixed for each data set and each value of $\gamma$; the points in the graph represent the average of three instances with different cardinalities. We observe that the screening procedure is able to fix most of the variables for $\gamma \leq 2^6\gamma_0$.  
As $\gamma$ increases further, the strength of the perspective relaxation decreases and the screening procedure is unable to fix as many variables. 

\begin{figure}[!h]
	\begin{center}
		\includegraphics[trim={10cm 6cm 10cm 6cm},clip,width=\columnwidth]{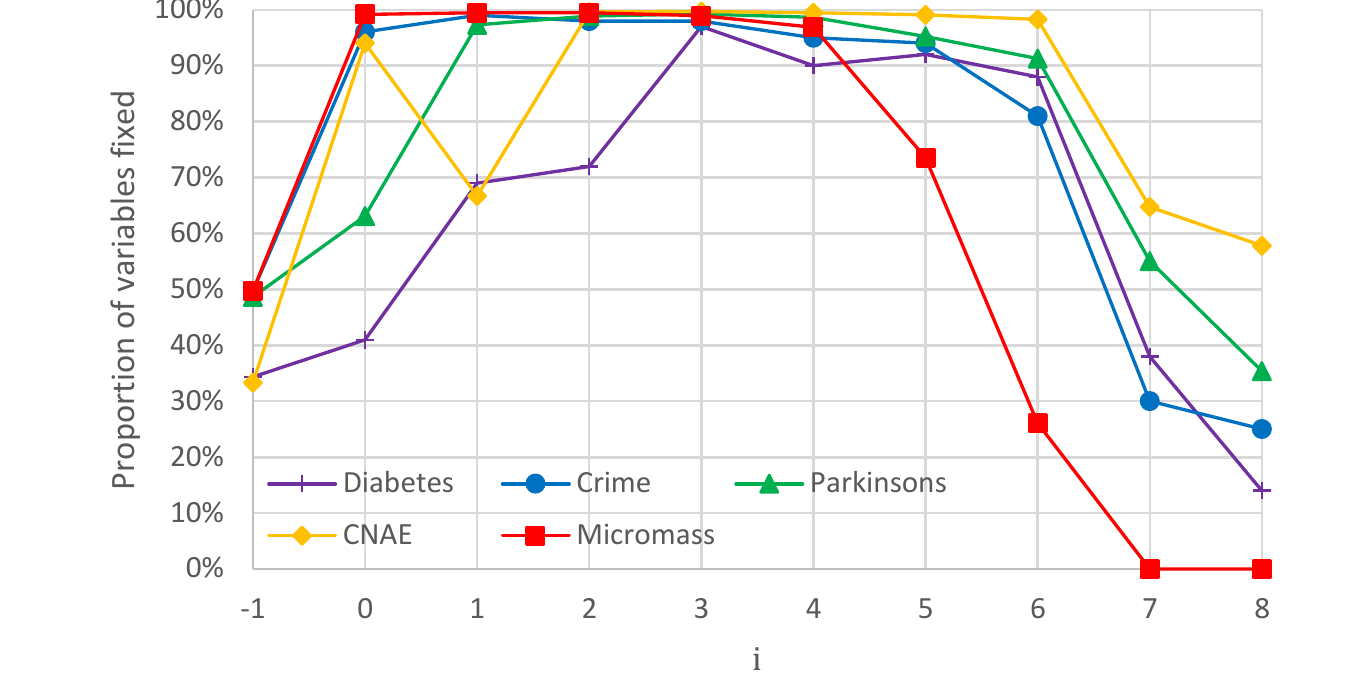}
		\caption{Proportion of variables fixed in instances with real data, where $\gamma=2^i\gamma_0$. Each point is an average of three instances with different cardinalities $k$.}
		\label{fig:fixed_real}
	\end{center}	
\end{figure}

Finally, Table~\ref{tab:small} shows four instances with the Diabetes data set where the screening procedure is able to fix only a small percentage of the variables, yet it results in substantial reduction in solution times\footnote{Instances with $k=10$ on this dataset are solved in five seconds or less independently of the use of the screening procedure, and are omitted. Similarly, instances with $\gamma\geq 2^7\gamma_0$ are solved in under 6 seconds, independent of the use of the screening procedure. Instances on datasets with $n\geq 100$ are rarely solved to optimality unless at least 50\% of the variables are fixed.}. The table shows the time in seconds and the number of branch-and-bound nodes required to solve the problems to optimality, and the \% of variables fixed by the screening procedure. Observe that even by fixing fewer than 20\% of the variables, the screening rule leads to a substantial reduction in running times. In some cases, instances that are not solved to optimality within the one-hour time limit are solved in under 15 seconds with screening.

\begin{table}[!h]
	\begin{center}
		\setlength{\tabcolsep}{3pt}
		\caption{Sample instances with the Diabetes dataset illustrating impact of fixing a small number of variables.}
		\label{tab:small}
		\begin{tabular}{c c | c c |c c c}
			\hline
			\multirow{2}{*}{$\bm{k}$}&\multirow{2}{*}{$\bm{\gamma}$}&\multicolumn{2}{c|}{\textbf{CPLEX}}&\multicolumn{3}{c}{\textbf{CPLEX+screening}}\\
			&   & \textbf{time} & \textbf{nodes}& \textbf{\% fixed} & \textbf{time} & \textbf{nodes}\\
			\hline
			20 & $\gamma_0/2$& 968 & 48,050 & 10.9\% & 303 & 10,552\\
			20 & $\gamma_0$& 2,080 & 80,095 & 14.1\% & $<$1 & 0\\
			30 &  $\gamma_0/2$& 2,791 & 119,638 & 20.3\% & 444 & 30,903\\
			30 &  $\gamma_0$& \text{1hr limit} & 168,311 & 9.4\% & 12 & 272\\
			\hline
		\end{tabular}
	\end{center}
\end{table}

\section{Conclusion}
\label{sec:conc}

We give a simple, yet very effective safe screening procedure for non-convex $\ell_0$ regression problems.
Computational on synthetic and real data sets show that
when used as preprocessing before solving the problems, the screening rules eliminate, on average, 76\% of
the binary variables, and consequently lead to substantial reduction in solution times.
Additional research on strong convex relaxations should lead to even more effective safe screening rules.

\section*{Acknowledgments}

Andr\'es G\'omez is supported, in part, by the National Science Foundation under Grant 1818700.
Alper Atamt\"urk is supported, in part, by grants from ARPA-E, NSF, and ONR.

\bibliography{Bibliography}

\begin{thebibliography}{40}
\providecommand{\natexlab}[1]{#1}
\providecommand{\url}[1]{\texttt{#1}}
\expandafter\ifx\csname urlstyle\endcsname\relax
  \providecommand{\doi}[1]{doi: #1}\else
  \providecommand{\doi}{doi: \begingroup \urlstyle{rm}\Url}\fi

\bibitem[Akt{\"u}rk et~al.(2009)Akt{\"u}rk, Atamt{\"u}rk, and
  G{\"u}rel]{akturk2009strong}
Akt{\"u}rk, M.~S., Atamt{\"u}rk, A., and G{\"u}rel, S.
\newblock A strong conic quadratic reformulation for machine-job assignment
  with controllable processing times.
\newblock \emph{Operations Research Letters}, 37\penalty0 (3):\penalty0
  187--191, 2009.

\bibitem[Atamt\"urk \& G\'omez(2019)Atamt\"urk and
  G\'omez]{atamturk2019rank-one}
Atamt\"urk, A. and G\'omez, A.
\newblock Rank-one convexification for sparse regression.
\newblock \emph{arXiv preprint arXiv:1901.10334}, 2019.

\bibitem[Atamt{\"u}rk et~al.(2000)Atamt{\"u}rk, Nemhauser, and
  Savelsbergh]{atamturk2000conflict}
Atamt{\"u}rk, A., Nemhauser, G.~L., and Savelsbergh, M.~W.
\newblock Conflict graphs in solving integer programming problems.
\newblock \emph{European Journal of Operational Research}, 121\penalty0
  (1):\penalty0 40--55, 2000.

\bibitem[Bertsimas \& Van~Parys(2017)Bertsimas and
  Van~Parys]{bertsimas2017sparse}
Bertsimas, D. and Van~Parys, B.
\newblock Sparse high-dimensional regression: Exact scalable algorithms and
  phase transitions.
\newblock \emph{arXiv preprint arXiv:1709.10029}, 2017.

\bibitem[Bertsimas et~al.(2016)Bertsimas, King, Mazumder,
  et~al.]{bertsimas2016best}
Bertsimas, D., King, A., Mazumder, R., et~al.
\newblock Best subset selection via a modern optimization lens.
\newblock \emph{The Annals of Statistics}, 44\penalty0 (2):\penalty0 813--852,
  2016.

\bibitem[Bertsimas et~al.(2019)Bertsimas, Pauphilet, and
  Van~Parys]{bertsimas2019sparse}
Bertsimas, D., Pauphilet, J., and Van~Parys, B.
\newblock Sparse regression: Scalable algorithms and empirical performance.
\newblock \emph{arXiv preprint arXiv:1902.06547}, 2019.

\bibitem[Cozad et~al.(2014)Cozad, Sahinidis, and Miller]{cozad2014learning}
Cozad, A., Sahinidis, N.~V., and Miller, D.~C.
\newblock Learning surrogate models for simulation-based optimization.
\newblock \emph{AIChE Journal}, 60\penalty0 (6):\penalty0 2211--2227, 2014.

\bibitem[Dedieu et~al.(2020)Dedieu, Hazimeh, and Mazumder]{dedieu2020learning}
Dedieu, A., Hazimeh, H., and Mazumder, R.
\newblock Learning sparse classifiers: Continuous and mixed integer
  optimization perspectives.
\newblock \emph{arXiv preprint arXiv:2001.06471}, 2020.

\bibitem[Dong et~al.(2015)Dong, Chen, and Linderoth]{dong2015regularization}
Dong, H., Chen, K., and Linderoth, J.
\newblock Regularization vs. relaxation: A conic optimization perspective of
  statistical variable selection.
\newblock \emph{arXiv preprint arXiv:1510.06083}, 2015.

\bibitem[Dua \& Graff(2017)Dua and Graff]{dua2019}
Dua, D. and Graff, C.
\newblock {UCI} machine learning repository, 2017.
\newblock URL \url{http://archive.ics.uci.edu/ml}.

\bibitem[Efron et~al.(2004)Efron, Hastie, Johnstone, Tibshirani,
  et~al.]{efron2004least}
Efron, B., Hastie, T., Johnstone, I., Tibshirani, R., et~al.
\newblock Least angle regression.
\newblock \emph{The Annals of Statistics}, 32\penalty0 (2):\penalty0 407--499,
  2004.

\bibitem[Efroymson(1966)]{efroymson1966stepwise}
Efroymson, M.
\newblock Stepwise regression--a backward and forward look.
\newblock \emph{Florham Park, New Jersey}, 1966.

\bibitem[El~Ghaoui \& Lebret(1997)El~Ghaoui and Lebret]{elghaoui1997robust}
El~Ghaoui, L. and Lebret, H.
\newblock Robust solutions to least-squares problems with uncertain data.
\newblock \emph{SIAM Journal on matrix analysis and applications}, 18\penalty0
  (4):\penalty0 1035--1064, 1997.

\bibitem[El~Ghaoui et~al.(2010)El~Ghaoui, Viallon, and
  Rabbani]{elghaoui2010safe}
El~Ghaoui, L.~E., Viallon, V., and Rabbani, T.
\newblock Safe feature elimination for the lasso and sparse supervised learning
  problems.
\newblock \emph{arXiv preprint arXiv:1009.4219}, 2010.

\bibitem[Fercoq et~al.(2015)Fercoq, Gramfort, and Salmon]{fercoq2015mind}
Fercoq, O., Gramfort, A., and Salmon, J.
\newblock Mind the duality gap: Safer rules for the lasso.
\newblock \emph{arXiv preprint arXiv:1505.03410}, 2015.

\bibitem[Frangioni \& Gentile(2006)Frangioni and
  Gentile]{frangioni2006perspective}
Frangioni, A. and Gentile, C.
\newblock Perspective cuts for a class of convex 0--1 mixed integer programs.
\newblock \emph{Mathematical Programming}, 106\penalty0 (2):\penalty0 225--236,
  2006.

\bibitem[Friedman et~al.(2010)Friedman, Hastie, and Tibshirani]{friedman2010}
Friedman, J., Hastie, T., and Tibshirani, R.
\newblock Regularization paths for generalized linear models via coordinate
  descent.
\newblock \emph{Journal of Statistical Software}, 33\penalty0 (1):\penalty0
  1--22, 2010.
\newblock URL \url{http://www.jstatsoft.org/v33/i01/}.

\bibitem[G{\'o}mez \& Prokopyev(2018)G{\'o}mez and Prokopyev]{gomez2018mixed}
G{\'o}mez, A. and Prokopyev, O.
\newblock A mixed-integer fractional optimization approach to best subset
  selection.
\newblock \emph{Optimization-online}, 2018.

\bibitem[G{\"u}nl{\"u}k \& Linderoth(2010)G{\"u}nl{\"u}k and
  Linderoth]{gunluk2010perspective}
G{\"u}nl{\"u}k, O. and Linderoth, J.
\newblock Perspective reformulations of mixed integer nonlinear programs with
  indicator variables.
\newblock \emph{Mathematical Programming}, 124\penalty0 (1-2):\penalty0
  183--205, 2010.

\bibitem[Hastie et~al.(2001)Hastie, Tibshirani, and
  Friedman]{hastie2001elements}
Hastie, T., Tibshirani, R., and Friedman, J.
\newblock \emph{The elements of statistical learning: Data mining, inference,
  and prediction}, volume~1.
\newblock Springer series in statistics New York, NY, USA, 2001.

\bibitem[Hastie et~al.(2015)Hastie, Tibshirani, and
  Wainwright]{hastie2015statistical}
Hastie, T., Tibshirani, R., and Wainwright, M.
\newblock \emph{Statistical learning with sparsity: The lasso and
  generalizations}.
\newblock CRC press, 2015.

\bibitem[Hastie et~al.(2017)Hastie, Tibshirani, and
  Tibshirani]{hastie2017extended}
Hastie, T., Tibshirani, R., and Tibshirani, R.~J.
\newblock Extended comparisons of best subset selection, forward stepwise
  selection, and the lasso.
\newblock \emph{arXiv preprint arXiv:1707.08692}, 2017.

\bibitem[Hazimeh \& Mazumder(2018)Hazimeh and Mazumder]{hazimeh2018fast}
Hazimeh, H. and Mazumder, R.
\newblock Fast best subset selection: Coordinate descent and local
  combinatorial optimization algorithms.
\newblock \emph{arXiv preprint arXiv:1803.01454}, 2018.

\bibitem[Hoerl \& Kennard(1970)Hoerl and Kennard]{hoerl1970ridge}
Hoerl, A.~E. and Kennard, R.~W.
\newblock Ridge regression: Biased estimation for nonorthogonal problems.
\newblock \emph{Technometrics}, 12\penalty0 (1):\penalty0 55--67, 1970.

\bibitem[Kimura \& Waki(2018)Kimura and Waki]{kimura2018minimization}
Kimura, K. and Waki, H.
\newblock Minimization of akaike's information criterion in linear regression
  analysis via mixed integer nonlinear program.
\newblock \emph{Optimization Methods and Software}, 33\penalty0 (3):\penalty0
  633--649, 2018.

\bibitem[Mazumder et~al.(2017)Mazumder, Radchenko, and
  Dedieu]{mazumder2017subset}
Mazumder, R., Radchenko, P., and Dedieu, A.
\newblock Subset selection with shrinkage: Sparse linear modeling when the snr
  is low.
\newblock \emph{arXiv preprint arXiv:1708.03288}, 2017.

\bibitem[Miller(2002)]{miller2002subset}
Miller, A.
\newblock \emph{Subset selection in regression}.
\newblock CRC Press, 2002.

\bibitem[Miyashiro \& Takano(2015)Miyashiro and Takano]{miyashiro2015subset}
Miyashiro, R. and Takano, Y.
\newblock Subset selection by {M}allows’ {C}p: A mixed integer programming
  approach.
\newblock \emph{Expert Systems with Applications}, 42\penalty0 (1):\penalty0
  325--331, 2015.

\bibitem[Ogawa et~al.(2013)Ogawa, Suzuki, and Takeuchi]{ogawa2013safe}
Ogawa, K., Suzuki, Y., and Takeuchi, I.
\newblock Safe screening of non-support vectors in pathwise {SVM} computation.
\newblock In \emph{International Conference on Machine Learning}, pp.\
  1382--1390, 2013.

\bibitem[Park \& Klabjan(2017)Park and Klabjan]{park2017subset}
Park, Y.~W. and Klabjan, D.
\newblock Subset selection for multiple linear regression via optimization.
\newblock \emph{arXiv preprint arXiv:1701.07920}, 2017.

\bibitem[Pilanci et~al.(2015)Pilanci, Wainwright, and
  El~Ghaoui]{pilanci2015sparse}
Pilanci, M., Wainwright, M.~J., and El~Ghaoui, L.
\newblock Sparse learning via boolean relaxations.
\newblock \emph{Mathematical Programming}, 151\penalty0 (1):\penalty0 63--87,
  2015.

\bibitem[Savelsbergh(1994)]{martin-prep}
Savelsbergh, M. W.~P.
\newblock Preprocessing and probing techniques for mixed integer programming
  problems.
\newblock \emph{ORSA J. on Computing}, 6\penalty0 (4):\penalty0 445--454, 1994.

\bibitem[Tibshirani(1996)]{tibshirani1996regression}
Tibshirani, R.
\newblock Regression shrinkage and selection via the lasso.
\newblock \emph{Journal of the Royal Statistical Society: Series B
  (Methodological)}, 58\penalty0 (1):\penalty0 267--288, 1996.

\bibitem[Tibshirani et~al.(2012)Tibshirani, Bien, Friedman, Hastie, Simon,
  Taylor, and Tibshirani]{tibshirani2012strong}
Tibshirani, R., Bien, J., Friedman, J., Hastie, T., Simon, N., Taylor, J., and
  Tibshirani, R.~J.
\newblock Strong rules for discarding predictors in lasso-type problems.
\newblock \emph{Journal of the Royal Statistical Society: Series B (Statistical
  Methodology)}, 74\penalty0 (2):\penalty0 245--266, 2012.

\bibitem[Wang et~al.(2013)Wang, Zhou, Wonka, and Ye]{wang2013lasso}
Wang, J., Zhou, J., Wonka, P., and Ye, J.
\newblock Lasso screening rules via dual polytope projection.
\newblock In \emph{Advances in Neural Information Processing Systems}, pp.\
  1070--1078, 2013.

\bibitem[Xiang \& Ramadge(2012)Xiang and Ramadge]{xiang2012fast}
Xiang, Z.~J. and Ramadge, P.~J.
\newblock Fast lasso screening tests based on correlations.
\newblock In \emph{2012 IEEE International Conference on Acoustics, Speech and
  Signal Processing (ICASSP)}, pp.\  2137--2140. IEEE, 2012.

\bibitem[Xiang et~al.(2016)Xiang, Wang, and Ramadge]{xiang2016screening}
Xiang, Z.~J., Wang, Y., and Ramadge, P.~J.
\newblock Screening tests for lasso problems.
\newblock \emph{IEEE Transactions on Pattern Analysis and Machine
  Intelligence}, 39\penalty0 (5):\penalty0 1008--1027, 2016.

\bibitem[Xie \& Deng(2020)Xie and Deng]{xie2020scalable}
Xie, W. and Deng, X.
\newblock Scalable algorithms for the sparse ridge regression.
\newblock 2020.

\bibitem[Xu et~al.(2009)Xu, Caramanis, and Mannor]{xu2009robustness}
Xu, H., Caramanis, C., and Mannor, S.
\newblock Robustness and regularization of support vector machines.
\newblock \emph{Journal of machine learning research}, 10\penalty0
  (Jul):\penalty0 1485--1510, 2009.

\bibitem[Zou \& Hastie(2005)Zou and Hastie]{zou2005regularization}
Zou, H. and Hastie, T.
\newblock Regularization and variable selection via the elastic net.
\newblock \emph{Journal of the Royal Statistical Society: Series B
  (Methodology)}, 67\penalty0 (2):\penalty0 301--320, 2005.

\end{thebibliography}
\bibliographystyle{icml2020}

\end{document}